\documentclass[letterpaper]{article}
\usepackage{uai2019}
\usepackage[margin=1in]{geometry}

\usepackage{times}
\usepackage{subcaption}
\usepackage{graphicx}

\usepackage{float}
\usepackage{amsmath,amsthm,amssymb}
\usepackage{amsfonts}
\usepackage{natbib}
\usepackage{dsfont}
\usepackage{algorithm}
\usepackage{algorithmic}
\usepackage[dvipsnames]{xcolor}
\usepackage[]{color-edits}
\usepackage{xspace}
\usepackage{enumitem}

\addauthor{AA}{Blue}
\addauthor{AS}{Red}
\addauthor{EB}{OliveGreen}
\addauthor{YL}{Cerulean}

\newcommand{\wmax}{\ensuremath{\sigma_w}}
\newcommand{\gmax}{\ensuremath{G_{\max}}}
\newcommand{\vmax}{\ensuremath{V_{\max}}}
\newcommand{\numin}{\ensuremath{\nu_{\min}}}

\newtheorem{theorem}{Theorem}

\newtheorem{assumption}{Assumption}
\newtheorem{definition}{Definition}
\newcommand{\goffpac}{\ensuremath{g_{\text{OffPAC}}}}
\newcommand{\offpac}{Off-PAC\xspace}
\newcommand{\R}{\ensuremath{\mathbb{R}}}
\newcommand{\E}{\ensuremath{\mathbb{E}}}
\newcommand{\inner}[2]{\ensuremath{\langle #1, #2 \rangle}}
\newcommand{\norm}[1]{\ensuremath{\left\| #1 \right\|}}
\newcommand{\grad}{\ensuremath{\zeta}}
\newcommand{\wh}{\ensuremath{\hat{w}}}
\newcommand{\Qh}{\ensuremath{\widehat{Q}}}
\newcommand{\eps}{\ensuremath{\varepsilon}}
\newcommand{\alg}{OPPOSD\xspace}

\usepackage{xr-hyper}
\usepackage{hyperref}
\usepackage{xr}

\makeatletter
\newcommand*{\addFileDependency}[1]{
  \typeout{(#1)}
  \@addtofilelist{#1}
  \IfFileExists{#1}{}{\typeout{No file #1.}}
}
\makeatother

\title{Off-Policy Policy Gradient with State Distribution Correction}

\author{
Yao Liu \\
Stanford University\\
\texttt{yaoliu@stanford.edu} \\
\And
Adith Swaminathan \\
Microsoft Research \\
\texttt{adswamin@microsoft.com} \\
\AND
Alekh Agarwal \\
Microsoft Research \\
\texttt{alekha@microsoft.com} \\
\And
Emma Brunskill \\
Stanford University\\
\texttt{ebrun@cs.stanford.edu}
}

\begin{document}

\maketitle

\begin{abstract}
We study the problem of off-policy
policy optimization in Markov decision processes, and develop a novel off-policy policy gradient method. 
Prior off-policy policy gradient approaches have generally ignored the mismatch between the distribution of
states visited under the behavior policy used to collect data, and what would
be the distribution of states under the learned policy. 
Here we build on recent
progress for estimating the ratio of the state distributions under behavior and evaluation policies for 
policy evaluation, and present an off-policy policy gradient optimization technique
that can account for this mismatch in distributions. We present an illustrative example
of why this is important and a theoretical convergence guarantee for our approach. Empirically, we compare our method in simulations to several strong baselines which do not correct for this mismatch, significantly improving in the quality of the policy discovered.
\end{abstract}

\section{INTRODUCTION}
\label{sec:intro}


The ability to use data about prior decisions and their outcomes to make counterfactual inferences about how alternative decision policies might perform, is a cornerstone of intelligent behavior. It also has immense practical potential -- it can enable the use of electronic medical record data to infer better treatment decisions for patients, the use of prior product recommendations to inform more effective strategies for presenting recommendations, and previously collected data from students using educational software to better teach those and future students. Such counterfactual reasoning, particularly when one is deriving decision policies that will be used to make not one but a sequence of decisions, is important since online sampling during a learning procedure is both costly and dangerous, and not practical in many of the applications above. While amply motivated, doing such counterfactual reasoning is also challenging because the data is censored -- we can only observe the result of providing a particular chemotherapy treatment policy to a particular patient, not the counterfactual of if we were then to start with a radiation sequence.


We focus on the problem of performing such counterfactual inferences in the context of sequential decision making in a Markov decision process (MDP). We assume that data has been previously collected using some fixed and known behavior policy, and our goal is to learn a new decision policy with good performance for future use. This problem is often known as batch off-policy policy optimization. We assume that the behavior policy used
to gather the data is stochastic: if it is deterministic, without any additional assumptions, we will not be able to estimate the
performance of any other policy.

In this paper we consider how to perform batch off-policy policy optimization (OPPO) using a
policy gradient method. While there has been increasing interest in batch off-policy reinforcement
learning (RL) over the last few years ~\citep{Thomas2015,Jiang2016,Thomas2016}, much
of this has focused on off-policy policy evaluation, where the goal is to estimate the performance
of a particular given target decision policy. Ultimately we will very frequently be interested in
the optimization question, which requires us to determine a good new policy for future potential
deployment, given a fixed batch of prior data.

To do batch off-policy policy optimization, value function methods (like Deep Q-Network~\citep{DQN}
or Fitted Q-Iteration~\citep{ernst2005tree}) can be used alone, but there are many cases where we might prefer
to focus on policy gradient or actor-critic methods.  Policy gradient methods have seen substantial success
in the last few years~\citep{schulman2015trust} in the on-policy setting, and they can be
particularly appealing for cases where it is easier to encode inductive bias in the policy space,
or when the actions are continuous (see e.g.~\citet{abbeel2016tutorial} for more discussion). 
However, existing approaches incorporating offline information into online policy gradients~\citep{gu46349,gu45838,Metelli2018Policy} have shown limited benefit,
in part due to the variance in gradients incurred due to incorporating off-policy data. One approach is to
correct exactly for the difference between the sampling data distribution and the target
policy data distribution, by using importance sampling to re-weight every sample according to the likelihood ratio of behavior policy and evaluation policy up to that step. Unfortunately the variance of this importance sampling ratio will grow exponentially with the problem horizon.

To avoid introducing variance in gradients, off-policy actor critic (Off-PAC)~\citep{degris2012off} ignores the state distribution difference between the behavior and target policies, and instead only uses a one step of importance sampling to reweight the action distributions.  Many practical off-policy policy optimization algorithms including DDPG \citep{silver2014deterministic}, ACER \citep{wang2016sample}, and Off-PAC with emphatic weightings~\citep{imani2018off} are based on the gradient expression in the Off-PAC algorithm \citep{degris2012off}. However as we will demonstrate, not correcting for this mismatch in state distributions can result in poor performance in general, both in theory and empirically. 

Instead, here we introduce an off-policy policy gradient algorithm that can be used with batch data
and that accounts for the difference in the state distributions between the current target and
behavior policies during each gradient step. Our approach builds on recent approaches for policy evaluation that avoid the exponential blow up in importance sampling
weights by instead directly computing the ratio of the distributions of state visitation
under the target and behavior policies~\citep{hallak2017consistent,liu2018breaking,gelada2019off}. 
We incorporate these ideas within an off-policy actor critic
method to do batch policy optimization. We first provide an illustrative example to demonstrate the benefit of this
approach over Off-PAC~\citep{degris2012off}, and show that correcting for the mismatch in
state distributions of the behavior and target policies can be critical for
getting good estimates of the policy gradient, and we also provide convergence guarantees
for our algorithm under certain assumptions. We then compare our approach and Off-PAC
experimentally on two simulated
domains, cart pole and a HIV patient simulator~\citep{ernst2005tree}. Our results show
that our approach is able to learn a substantially higher performing policy than
both Off-PAC and the behavior policy that is used to gather the batch data. 
While they have not been previously evaluated in the batch setting, we also compare our method against two more recent policy optimization approaches which also make use of available off-policy data. We find that without the several enhancements on top of vanilla actor-critic method which these methods incorporate, we are able to outperform them in our simulations by virtue of the state distribution correction. We further
demonstrate that we can use the recently proposed off-policy evaluation technique of~\citet{liu2018breaking} to reliably identify good policies
found during the policy gradient optimization run. Our results suggest that
directly accounting for the state distribution mismatch can be done without
prohibitively increasing the variance during policy gradient evaluations, and that doing so can yield significantly
better policies. These results are promising for enabling us to learn better
policies given batch data or improving the sample efficiency of online policy
gradient methods by being able to better incorporate past data.

\paragraph{Related Work} 
Many prior works focus on the off-policy policy evaluation (OPPE) problem, as it is a foundation for downstream policy learning problems. These approaches often build on importance sampling techniques to correct for distribution mismatch in the trajectory space, pioneered by the early work on eligibility traces~\citep{precup2000eligibility}, and further enhanced with a variety of variance reduction techniques ~\citep{Thomas2015, Jiang2016, Thomas2016}. Some consider model-based approaches to OPPE~\citep{farajtabar2018more,liu2018representation}, which usually perform better than importance sampling approaches empirically in policy evaluation settings. But those methods do not extend easily to our OPPO setting, as well as introduce additional challenges due to bias in the models and typically require fitting a separate model for each target policy. The recent work of~\citet{liu2018breaking} partially alleviates the variance problem for model-free OPPE by reweighting the state visitation distributions, which can result in as just as high a variance in the worst case, but is often much smaller. Our work incorporates this recent estimator in policy optimization methods to enable learning from off-policy collected data.

In the off-policy policy optimization setting, many works study value-function based approaches (like fitted Q iteration~\citep{ernst2005tree} and DQN~\citep{DQN}), as they are known to be more robust to distribution mismatch. Some recent works aim to further incorporate reweighting techniques within off-policy value function learning~\citep{hallak2017consistent,gelada2019off}. These methods hint at the intriguing potential of value-function based techniques for off-policy learning, and we are interested in similarly understanding the viability of using direct policy optimization techniques in the off-policy setting.

Off-policy actor critic method \citep{degris2012off, imani2018off} propose a partial answer to this question by learning the critic using off-policy data and reweighting actor gradients by correcting the conditional action probabilities, but ignore the mismatch between the state distributions. A different research thread on trust region policy optimization method \citep{schulman2015trust}, while requiring the on-policy setting, incorporates robustness to the mismatch between the data collection and gradient evaluation policies. However this is not a fully off-policy scenario and learning from an offline dataset is still strongly motivated by many applications. Somewhat related, many recent off-policy RL algorithms \citep{silver2014deterministic, wang2016sample, gu45838, gu46349, lillicrap2015continuous, Haarnoja2018Soft, Dai2018SBEED} improve the empirical sample efficiency by using more off-policy samples from the previous iterations. In contrast to our setting with a fixed dataset collected with a behavior policy, these methods are online off-policy in that they still collect data with each policy found in the optimization procedure, but also use previous data for added sample efficiency. We nevertheless compare to two such methods in our empirical evaluation.

\section{PRELIMINARIES}
\label{sec:prelim}

We consider finite horizon MDPs $M = \langle\mathcal{S}, \mathcal{A}, P, r, \gamma\rangle$, with a continuous state space $\mathcal{S}$, a discrete action space $\mathcal{A}$, a transition probability distribution $P: \mathcal{S} \times \mathcal{A} \times \mathcal{S} \mapsto [0,1]$ and an expected reward function $r: \mathcal{S} \times \mathcal{A} \mapsto [0,1]$. We observe tuples of state, action, reward and next state: $(s_t, a_t, r_t, s_{t+1})$, where $s_0$ is drawn from a initial state distribution $p_0(s)$, action $a$ is drawn from a stochastic behavior policy $\mu(a|s)$ and the reward and next state are generated by the MDP. Given a discount factor $\gamma \in (0,1]$, the goal is to maximize the expected return of policy:
\begin{align}
  R^{\pi}_M = \mathbb{E}_{\pi} \left[ \lim_{T\to\infty} \frac{1}{\sum_{t=0}^T \gamma^t} \sum_{t=0}^{T} \gamma^t r_t \right]
\end{align}
When $\gamma = 1$ this becomes the average reward case and $\gamma < 1$ is called the discounted reward case. In this paper, we focus on $\gamma < 1$ for most of the results. Given any fixed policy $\pi$ the MDP becomes a Markov chain and we can define the state distribution at time step $t$: $d^{\pi}_{t}(s)$, and the stationary state distribution across time:
$
    d^{\pi}(s) = \displaystyle\lim_{T\to\infty}\textstyle\frac{1}{\sum_{t=0}^T \gamma^t} \sum_{t=0}^{T} \gamma^t d^{\pi}_{t}(s)
$
We consider infinite horizon for the convenience of defining stationary distribution. A similar scheme should work for finite horizon, but requires time-dependent state distribution and we omit this here. To make sure the optimal policy is learnable from collected data, we assume the following about the support set of behavior policy:
\begin{assumption}
\label{assum:coverage}
For at least one optimal policy $\pi^*$, $d^{\mu}(s) > 0$ for all $s$ such that $d^{\pi^*}(s) > 0$, and $\mu(a|s) > 0$ for all $a$ such that $\pi^*(a|s) > 0$ when $d^{\pi^*}(s) > 0$.
\end{assumption}

\section{AN OFF-POLICY POLICY GRADIENT ESTIMATOR}
Note that Assumption~\ref{assum:coverage} is quite weak when designing a policy evaluation or optimization scheme, since it only guarantees that $\mu$ adequately visits all the states and actions visited by some $\pi^*$. However, a policy optimization algorithm might require off-policy policy gradient estimates at arbitrary intermediate policy it produces along the way, which might visit states not reached by $\mu$. A strong assumption to handle such scenarios is that Assumption~\ref{assum:coverage} holds not just for some $\pi^*$, but any possible policy $\pi$. Instead of making such a strong assumption, we start by defining an augmented MDP where Assumption~\ref{assum:coverage} suffices for obtaining pessimistic estimates of policy values and gradients.

\paragraph{Constructing an Augmented MDP} 
Given a data collection policy $\mu$, let its support set be $\mathcal{S}_{\mu} = \{ s: d^{\mu}(s) > 0\}$ and $\mathcal{SA}_{\mu} = \{(s,a): d^{\mu}(s)\mu(a|s) > 0\}$. Consider a modified MDP $M_{\mu} = \langle\mathcal{S}_{\mu}\bigcup\{s_{abs}\}, \mathcal{A}, P_{\mu}, r_{\mu}, \gamma\rangle$. Any state-action pairs not in $\mathcal{SA}_{\mu}$ will essentially transition to $s_{abs}$ which is a new absorbing state where all actions will lead to a zero reward self-loop. Concretely, $P_{\mu}(s_{abs}|s_{abs},a) = 1$ and $r(s_{abs},a) = 0$ for any $a$. For all other states, the transition probabilities and rewards are defined as:  For $(s,a) \in \mathcal{SA}_{\mu}$, $P_{\mu}(s'|s,a) = P(s'|s,a)$ for all $s' \in \mathcal{S}_{\mu}$, and $P_{\mu}(s_{abs}|s,a) = \int_{s \not\in \mathcal{S}_{\mu}}P(s'|s,a)ds'$. For all $s \in \mathcal{S}_{\mu}$ but $(s,a) \not\in \mathcal{SA}_{\mu}$, $P_{\mu}(s_{abs}|s,a)=1$. $r_{\mu}(s,a) = r(s,a)$ for $(s,a) \in \mathcal{SA}_{\mu}$, and $r_{\mu}(s,a) = 0$ otherwise. First we prove that the optimal policy $\pi^*$ of the original MDP remains optimal in augmented MDP as a consequence of Assumption~\ref{assum:coverage}. 

\begin{theorem}
\label{thm:augmented_mdp}
The expected return of all policies $\pi$ in the original MDP is larger than the expected return in the new MDP: $R^{\pi}_{M} \ge R^{\pi}_{M_{\mu}}$. For any optimal $\pi^*$ that satisfies Assumption \ref{assum:coverage} we have that $R^{\pi^*}_{M} = R^{\pi^*}_{M_{\mu}}$
\end{theorem}
That is, policy optimization in $M_\mu$ has at least one optimal solution identical to the original MDP $M$ with the same policy value since $M_\mu$ lower bounds the policy value in $M$, so sub-optimal policies remain sub-optimal. Proof in Appendix \ref{sec:proof_thm1_appendix}.

\paragraph{Off-Policy Policy Gradient in Augmented MDP}
We will now use the expected return in the modified MDP, $R^{\pi}_{M_{\mu}}$, as a surrogate for deriving policy gradients. According to the policy gradient theorem in \citet{sutton2000policy}, for a parametric policy $\pi$ with parameters $\theta$:
\begin{equation*}
    \frac{\partial R^{\pi}_{M_{\mu}}}{\partial \theta} = \mathbb{E}_{d^\pi} \left[  \int_{a} \pi(a|s) \frac{\partial \log \pi(a|s)}{\partial \theta}Q^{\pi}_{M_{\mu}}(s,a) da \right],
\end{equation*}

From here on, $d^{\pi}(s)$ is with respect to the new MDP. The definition of $Q^\pi_{M_\mu}$ follows \citet{sutton2000policy}. \footnote{For discounted case, our definition of expected return differs from the definition of \citet{sutton2000policy} by a normalization factor $(1-\gamma)$. This is because the definitions of stationary distributions are scaled differently in the two cases.}

Now we will show that we can get an unbiased estimator of this gradient using importance sampling from the stationary state distribution $d^{\mu}(s)$ and the action distribution $\mu(a|s)$. According to the definition of $M_{\mu}$, we have that for all ${s,a}$ such that $d^{\mu}(s)\mu(a|s) = 0$, $(s,a)$ is not in $\mathcal{SA}_{\mu}$. Hence $Q^{\pi}_{M_{\mu}}(s,a) = 0$ for any policy $\pi$ since $(s,a)$ will receive zero reward and lead to a zero reward self-loop. So we have:
\begin{align}
    & \frac{\partial R^{\pi}_{M_{\mu}}}{\partial \theta} = \sum_{s} d^{\pi}(s) \sum_{a} \pi(a|s) \frac{\partial \log \pi(a|s)}{\partial \theta}Q^{\pi}_{M_{\mu}}(s,a) \nonumber\\
    &= \sum_{s~:~d^{\mu}(s)>0} \frac{d^{\pi}(s)}{d^{\mu}(s)}d^{\mu}(s)  \nonumber \\
    &\sum_{a~:~\mu(a|s)>0} \frac{\pi(a|s)}{\mu(a|s)}\mu(a|s) \frac{\partial \log \pi(a|s)}{\partial \theta}Q^{\pi}_{M_{\mu}}(s,a) \nonumber\\
    &= \mathbb{E}_{d^\mu,\mu} \left[ \frac{d^{\pi}(s)}{d^{\mu}(s)}\frac{\pi(a|s)}{\mu(a|s)} \frac{\partial \log \pi(a|s)}{\partial \theta}Q^{\pi}_{M_{\mu}}(s,a) \right] \label{eqn:our_policy_gradient}
\end{align}

Note that according to the definition of $M_{\mu}$, the Markov chain induced by $M$ and $\mu$ is exactly the same as $M_{\mu}$ and $\mu$. Thus the distribution of $(s_t,a_t,s_{t+1})$ generated by executing $\mu$ in $M$ is the same as executing $\mu$ in $M_{\mu}$. So, we can estimate this policy gradient using the data we collected from $\mu$ in $M$. We conclude the section by pointing out that working in the augmented MDP allows us to construct a reasonable off-policy policy gradient estimator under the mild Assumption~\ref{assum:coverage}, while all prior works in this vein either explicitly or implicitly require the coverage of all possible policies.

Note that in the average reward case, such an augmented MDP would not be helpful for policy optimization since all policies that potentially reach $s_{abs}$ will have a value of zero, and the stationary state distribution will be a single mass in the absorbing state. That would not induce a practical policy optimization algorithm. In the average reward case, either we need a stronger assumption that $\mu$ covers the entire state-action space or we must approximate the problem by setting a discount factor $\gamma < 1$ for the policy optimization algorithm, which is a common approach for deriving practical algorithms in an average reward (episodic) environment.

\section{ALGORITHM: \alg}
\label{sec:opposd}

Given the off-policy policy gradient derived in \eqref{eqn:our_policy_gradient}, how can we efficiently estimate it from samples collected from $\mu$? Notice that most quantities in the gradient estimator~\eqref{eqn:our_policy_gradient} are quite intuitive and also present in prior works such as \offpac. The main difference is the state distribution reweighting $d^\pi(s)/d^\mu(s)$, which we would like to estimate using samples collected with $\mu$. For estimating this ratio of state distributions, we build on the recent work of~\citet{liu2018breaking} which we describe next.

For a policy $\pi$, let us define the shorthand $\rho_\pi(s,a) = \pi(s,a)/\mu(s,a)$. Further given a function $w~:~\mathcal{S}\to \mathbb{R}$, define $\Delta(w;s,a,s') := w(s) \rho_\pi(s,a) - w(s')$. Then we have the following result.

\begin{theorem}[\citep{liu2018breaking}]
Given any $\gamma \in (0,1)$, assume that $d^\mu(s) > 0$ for all $s$ and define
\begin{align*}
L(w,f) &= \gamma \E_{(s,a,s')\sim d^\mu}[\Delta(w;s,a,s')f(s')] \\ & + (1-\gamma)\E_{s\sim p_0}[(1-w(s))f(s)].
\end{align*}
Then $w(s) = d^\pi(s)/d^\mu(s)$ if and only if $L(w,f) = 0$ for any measurable test function $f$.\footnote{When $\gamma = 1$, $w$ is only determined up to normalization, and hence an additional constraint $\E_{s\sim d^\mu}[w(s)] = 1$ is required to obtain the conclusion $w(s) = d^\pi(s)/d^\mu(s)$.}
\label{thm:w-def}
\end{theorem}

This result suggests a constructive procedure for estimating the state distribution ratio using samples from $\mu$, by finding a function $w$ over the states which minimizes $\max_f L(w,f)$. Since the maximization over all measurable functions as per Theorem~\ref{thm:w-def} is intractable,~\citet{liu2018breaking} suggest restricting the maximization to a unit ball in an Reproducing Kernel Hilbert Space (RKHS), which has an analytical solution to the maximization problem, and we use the same procedure to approximate density ratios in our algorithm.

Applying Theorem~\ref{thm:w-def} requires overcoming one final obstacle. The theorem presupposes $d^\mu(s) > 0$ for all $s$. In case where $\mathcal{SA}_{\mu} = \mathcal{S}\times\mathcal{A}$ we can directly apply the theorem. Otherwise in the MDP $M_\mu$, this assumption indeed holds for all states, but $\mu$ never visits the absorbing state $s_{abs}$, or any transitions leading into this state. However, since we know this special state, as well as the dynamics leading in and out of it, we can simulate some samples for this state, effectively corresponding to a slight perturbation of $\mu$ to cover $s_{abs}$. Concretely, we first choose a small smoothing factor $\epsilon \in (0,1)$. For any sample $(s,a,s')$ in our data set, if there exist $k$ actions $\widetilde{a}$ such that $\mu(\widetilde{a}|s) = 0$, then we will keep the old samples with probability $1-\epsilon$ and sample any one of the $k$ actions with probability $\epsilon/k$ uniformly and change the next state $s'$ to $s_{abs}$. If we sampled $\widetilde{a}$, consequently, we would also change all samples after this transition to a self-loop in $s_{abs}$. Thus we create samples drawn according to a new behavior policy which covers all the state action pairs: $\widetilde{\mu} = (1-\epsilon)\mu + \epsilon U(s)$ where $U(s)$ is a uniform distribution over the $k$ actions not chosen by $\mu$ in state $s$. Now we can use Theorem \ref{thm:w-def} and the algorithm from \citet{liu2018breaking} to estimate $d^{\pi}(s)/d^{\widetilde{\mu}}(s)$. Note that the propensity scores and policy gradients computed on this new dataset correspond to the behaviour policy $\widetilde{\mu}$ and not $\mu$. Formally, in place of using \eqref{eqn:our_policy_gradient}, we now estimate:
\begin{equation}
    \mathbb{E}_{d^{\widetilde{\mu}},\widetilde{\mu}} \left[ \frac{d^{\pi}(s)}{d^{\widetilde{\mu}}(s)}\frac{\pi(a|s)}{\widetilde{\mu}(a|s)} \frac{\partial \log \pi(a|s)}{\partial \theta}Q^{\pi}_{M_{\mu}}(s,a) \right] \label{eqn:our_soften_policy_gradient}
\end{equation}
Note that we can estimate the expectation in \eqref{eqn:our_soften_policy_gradient} from the smoothed dataset by construction, since the ratio $\pi(s,a)/\widetilde{\mu}(s,a)$ in all states are known.

Now that we have an algorithm for estimating policy gradients from~\eqref{eqn:our_soften_policy_gradient}, we can plug this into any policy gradient optimization method. Following prior work, we incorporate our off-policy gradients into an actor-critic algorithm.  For learning the critic $Q^{\pi}_{M_{\mu}}(s,a)$, we can use any off-policy Temporal Difference \citep{bhatnagar2009convergent, maei2011gradient} or Q-learning algorithm \citep{watkins1992q}. In our algorithm, we fit an approximate value function $\widehat{V}$ by:
\footnote{For simplicity, Eqn 4 views $\widehat{V}(s)$ in the tabular setting. See Line 14 in Alg \ref{alg:opposd} for the function approximation case.}
\begin{equation}
    \widehat{V}(s) \leftarrow \widehat{V}(s) + \alpha_c \frac{\pi(a|s)}{\widetilde{\mu}(a|s)} \left( R^{\lambda}(s,a) - \widehat{V}(s) \right),
    \label{eqn:critic-update}
\end{equation}
where $\alpha_c$ is the step-size for critic updates and $R^{\lambda}(s,a)$ is the off-policy $\lambda$-return:
\begin{small}
\begin{align*}
    R^{\lambda}(s,a) = r(s,a) + (1-\lambda) \gamma \widehat{V}(s') + \lambda \gamma \frac{\pi(a|s)}{\widetilde{\mu}(a|s)} R^{\lambda}(s',a'),
\end{align*}
\end{small}
and $(s,a,s',a')$ is generated by executing $\widetilde{\mu}$. After we learn $\widehat{V}$, $R^\lambda$ serves the role of $Q^{\pi}_{M_{\mu}}$ in our algorithm. 

Given the estimates of the state distribution ratio from~\citet{liu2018breaking} and the critic updates from~\eqref{eqn:critic-update}, we can now update the policy by plugging these quantities in~\eqref{eqn:our_soften_policy_gradient}. It remains to specify the initial conditions to start the algorithm. Since we have data collected from a behavior policy, it is natural to also warm-start the actor policy in our algorithm to be the same as the behavior policy and correspondingly the critic and $w$'s to be the value function and distribution ratios for the behavior policy. This can be particularly useful in situations where the behavior policy, while suboptimal, still gets to states with high rewards with a reasonable probability. Hence we use behavior cloning to warm-start the policy parameters for the actor, use on-policy value function learning for the critic and also fit the state ratios $w$ for the actor obtained by behavior cloning. Note that while the ratio will be identically equal to 1 if our behavior cloning was perfect, we actually estimate the ratio to better handle imperfections in the learned actor initialization.

\begin{algorithm}[htb]
 \caption{OPPOSD: Off-Policy Policy Optimization with State Distribution Correction}
 \label{alg:opposd}
 \begin{algorithmic}[1]
 \REQUIRE $\mathcal{S}, \mathcal{A}, \mu, \mathcal{D}: \left\{ \{ s_t^{i},a_t^{i},r_t^{i},\mu(a_t^{i}|s_t^{i}) \}_{t=0}^{T} \right\}_{i=0}^{n}$
 \REQUIRE Hyperparameters $\lambda$, $\gamma$, $N_w$, $\alpha_w$, $N_c$, $\alpha_c$, $\alpha$
 \STATE Warm start $\pi_{\theta}$, $\widehat{V}_{\theta_c}$, $w_{\theta_w}$
 \STATE Pad $\mathcal{D}$ to get $\mathcal{D'}, \widetilde{\mu}$ if necessary
 \FOR{each step of policy update}
    \FOR{state ratio updates $i=1,2,\ldots,N_w$}
        \STATE Sample a mini-batch $B_w \sim \mathcal{D'}$ according to $\widehat{d}_\gamma$
        \footnotemark
        \IF{$\gamma=1$}
            \STATE Perform one update according to Algorithm 1 in \citet{liu2018breaking} with stepsize $\alpha_w$
        \ELSE
            \STATE Perform one update according to Algorithm 2 in \citet{liu2018breaking} with stepsize $\alpha_w$
        \ENDIF
    \ENDFOR
    \FOR{critic updates $i=1,2,\ldots,N_c$}
        \STATE Sample a mini-batch $B_c \sim \mathcal{D'}$
        \STATE $\theta_c \leftarrow \theta_c - \alpha_c \frac{\partial\ell_c}{\partial\theta_c} $, where:
        $\ell_c = \frac{1}{|B_c|}\sum_{(s,a,s') \sim B_c} \frac{\pi(a|s)}{\widetilde{\mu}(a|s)} \left( R^{\lambda}(s,a) - \widehat{V}(s) \right)^2$
    \ENDFOR
    \STATE Sample a mini-batch $B_a \sim \mathcal{D'}$ according to $\widehat{d}_\gamma$
    \STATE $z_w \leftarrow \frac{1}{|B_a|}\sum_{s \sim B_a} w(s)$
    \STATE $Q^{\pi}(s,a) \leftarrow \mathds{1}\{(s,a) \in \mathcal{SA}_{\mu}\} R^{\lambda}(s,a)$
    \STATE $\theta \leftarrow \theta - \frac{\alpha}{|B_a|}\sum_{s \sim B_a} \frac{w(s)}{z_w}\rho(s,a) \frac{\partial\log\pi(a|s)}{\partial\theta} Q^{\pi}(s,a)$
 \ENDFOR
 \end{algorithmic}
\end{algorithm}
\footnotetext{$\widehat{d}_{\gamma} = \frac{1}{\sum_{t=0}^T \gamma^t} \sum_{t=0}^{T} \gamma^t \widehat{d}_{t}(s)$, where $\widehat{d}_{t}(s)$ is the empirical state distribution at time step $t$ in dataset $\mathcal{D'}$}

A full pseudo-code of our algorithm, which we call \alg for Off-Policy Policy Optimization with State Distribution Correction, is presented in Algorithm~\ref{alg:opposd}. We mention a couple of implementation details which we found helpful in improving the convergence properties of the algorithm. Typical actor-critic algorithms update the critic once per actor update in the on-policy setting. However, in the off-policy setting, we find that performing multiple critic updates before an actor update is helpful, since the off-policy TD learning procedure can have a high variance. Secondly, the computation of the state distribution ratio $w$ is done in an online manner similar to the critic updates, and analogous to the critic, we always retain the state of the optimizer for $w$ across the actor updates (rather than learning the $w$ from scratch after each actor update). Similar to the critic, we also perform multiple $w$ updates after each actor update. These choices are intuitively reasonable as the standard two-time scale asymptotic analysis of actor-critic methods~\citep{borkar2009stochastic} does require the critic to converge faster than the actor.

\section{CONVERGENCE RESULT}
\label{sec:analysis}

In this section, we present two main results to demonstrate the theoretical advantage of our algorithm. First we present a simple scenario where the prior approach of \offpac yields an arbitrarily biased gradient estimate, despite having access to a perfect critic. In contrast \alg estimates the gradients correctly whenever the distribution ratios in~\eqref{eqn:our_policy_gradient} and the critic are estimated perfectly, by definition. We will further provide a convergence result for \alg to a stationary point of the expected reward function.

\paragraph{A hard example for Off-PAC} Many prior off-policy policy gradient methods use the policy gradient estimates proposed in \citet{degris2012off}.
\begin{equation*}
    \goffpac(\theta) = \sum_{s} d^{\mu}(s) \sum_{a} \pi(a|s) \frac{\partial \log\pi(a|s)}{\partial \theta} Q^{\pi}(s,a)
\end{equation*}
Notice that, in contrast to the exact policy gradient, the expectation over states is taken with respect to the behavior policy $\mu$ distribution instead of $\pi$. In tabular settings this can lead to correct policy updates, as proved by \citet{degris2012off}. We now present an example where the policy gradient computed this way is problematic when using function approximators. Consider the problem instance shown in Figure \ref{fig:hard_example}, where the behavior policy $\pi_b$ is uniformly random in any state. Now we consider policies parameterized by a parameter $\alpha \in [0,1]$ where $\pi_\alpha(s_1,\ell) =  \pi_\alpha(s_2, \ell) = \alpha$ and $\pi_\alpha(s,\ell) = 1$ for other states.
Thus $\pi_\alpha$ aliases the states $s_3$ and $s_4$ as a manifestation of imperfect representation which is typical with large state spaces. Clearly the optimal policy is $\pi_1$. With perfect critic, in Appendix \ref{sec:hard_example} we show that for \offpac the gradient vanishes for any policy $\pi_\alpha$, meaning that the algorithm can be arbitrarily sub-optimal at any point during policy optimization. Our gradient estimator~\eqref{eqn:our_policy_gradient} instead evaluates to $\partial R^{\pi_\alpha}_{M_\mu}/\partial\alpha = (1+\alpha)/4$, which is correctly maximized at $\alpha = 1$.


\begin{figure}[ht]
\centering
  \includegraphics[height=1.5in]{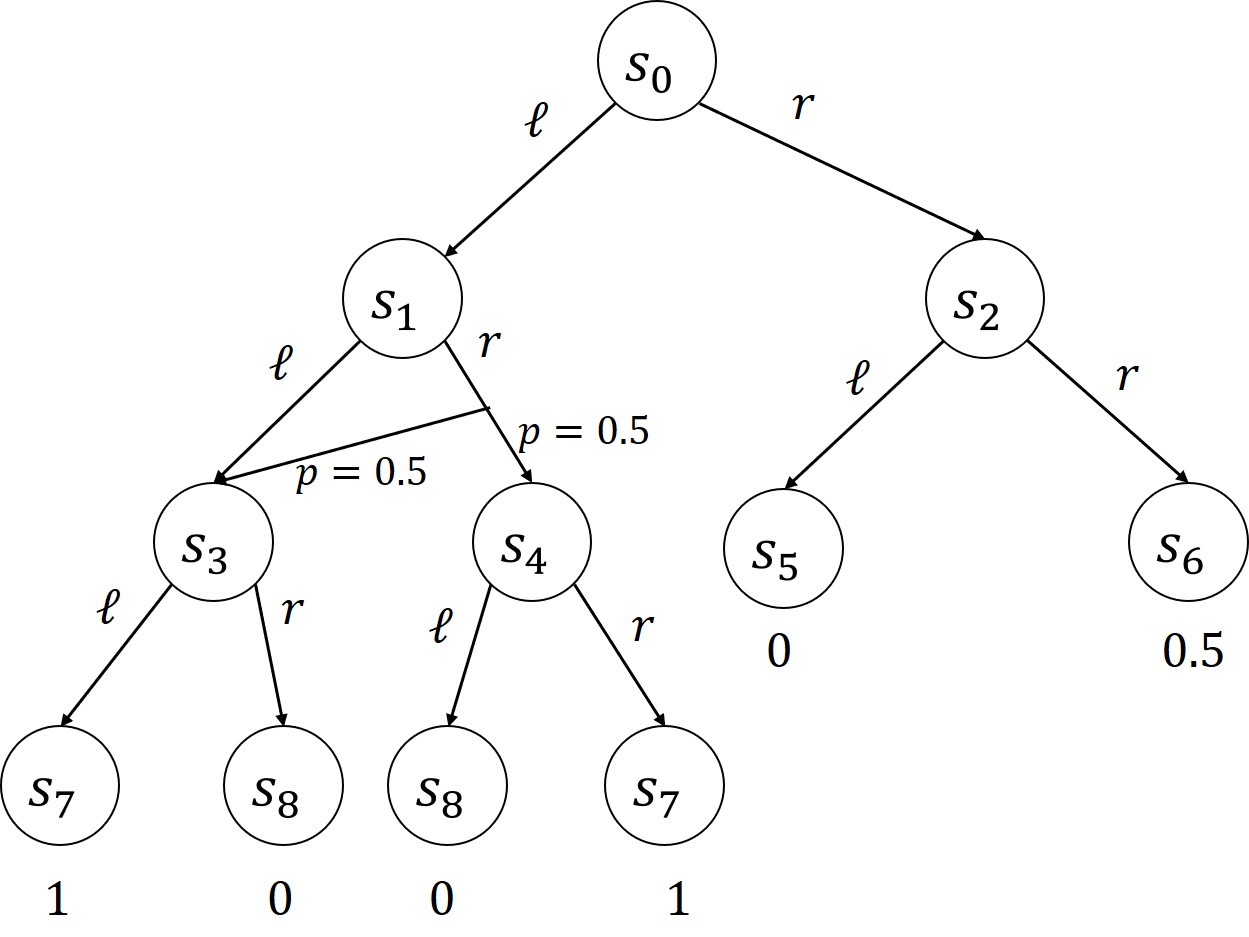}
    \caption{Hard example for Off-PAC~\citep{degris2012off}}
    \label{fig:hard_example}
\end{figure}


\paragraph{Convergence results for \alg} We next ask whether \alg converges, given reasonable estimates for the density ratio and the critic. To this end, we need to introduce some additional notations and assumptions. Suppose we run \alg over some parametric policy class $\pi_\theta$ with $\theta \in \Theta$. In the sequel, we use subscripts and superscripts by $\theta$ to mean the corresponding quantities with $\pi_\theta$ to ease the notation. We begin by describing an abstract set of assumptions and a general result about the convergence of \alg, when we run it over the policies $\pi_\theta$ given data collected with an arbitrary policy $\nu$, before instantiating our assumptions for the specific structure of 
$\widetilde{\mu}$ used in our algorithm.
\begin{definition}
A function $f~:~\R^d \to \R$ is $L$-smooth when
\[
    \|\nabla f(x) - \nabla f(y) \|_2 \leq L \| x - y\|_2, \quad \mbox{ for all $x,y \in \R^d$,}
\]
\end{definition}

\begin{assumption}
\label{assum:convergence}
$\forall (s,a)$ pairs, $\forall \theta \in \Theta$ and a data collection policy $\nu$, we assume that the MDP guarantees:
\begin{enumerate}[nolistsep]
    \item $\norm{\frac{\partial \pi_\theta(a|s)}{\partial\theta}} \le \gmax$.
    \item $Q^{\theta}(s,a) \leq \vmax$.
    \item $\nu(a|s) \ge \numin$.
    \item $w(s) := d^{\theta}(s)/d^\nu(s)$ has a bounded second moment $\wmax^2 := E_\nu [w(s)^2] \le \infty$.
    \item The expected return of $\pi_{\theta}$: $R^{\theta}$ is a differentiable, $G$-Lipschitz and $L$-smooth function w.r.t. $\theta$.
\end{enumerate}
\end{assumption}

\begin{theorem}
Assume an MDP, a data collection policy $\nu$ and function classes $\{ \pi_\theta \}$ and $\{ \wh \}$ satisfy Assumption \ref{assum:convergence}.
Suppose \alg with policy parameters $\theta_k$ at iteration $k$ is provided critic estimates $\Qh_k$ and distribution ratio estimates $\wh_k$  satisfying $\E_{(s,a)\sim d^\nu}(w_{\theta_k}(s,a) - \wh_k(s,a))^2 \leq \eps^2_{w,k}$ and $\E_{(s,a)\sim d^\nu}(Q^{\theta_k}(s,a) - \Qh_k(s,a))^2 \leq \eps^2_{Q,k}$ for iterations $k=1,2,\ldots K$. Then
\begin{align}
&\frac{1}{K}\sum_{k=1}^K \E\left[\norm{\nabla_\theta R^{\theta_k}}^2\right] \leq \frac{2\vmax}{K} + \nonumber \\
& \frac{\sum_{k=1}^K O\big( (\eps_{w,k}^2 \vmax^2 + \eps_{Q,k}^2 (\wmax^2+\eps_{w,k}^2))\gmax^2\big)}{K\numin^2} \nonumber 
\end{align}
\label{thm:convergence}
\end{theorem} 
That is, when Assumption \ref{assum:convergence} holds, the scheme converges to an approximate stationary point given estimators $\wh$ and $\Qh$ with a small average MSE across the iterations under $\nu$. An immediate consequence of the theorem above is that as long as we guarantee that $\lim_{K \to \infty} \frac{\sum \eps_{w,k}^2 + \eps_{Q,k}^2}{K} = 0$, which a reasonable online critic and $w$ learning algorithm can guarantee, we have:
$\lim_{K \to \infty} \frac{1}{K}\sum_{k=1}^K \E\left[\norm{\nabla_\theta R^{\theta_k}_{M_\nu}}^2\right] = 0$ which implies the procedure will converge to a stationary point where the true policy gradient is zero. 
Proof of this theorem in Appendix \ref{sec:proof_thm3_appendix}.

We now discuss the validity of Assumption~\ref{assum:convergence} in the specific context of the data collection policy $\widetilde{\mu}$ used in \alg as well as the augmented MDP $M_\mu$. The first assumption, that the gradient of policy distribution is bounded, can be achieved by an appropriate policy parametrization such as a linear or a differentiable neural network-based scoring function composed with a softmax link. The second assumption on bounded value functions is standard in the literature. In particular, both these assumptions are crucial for the convergence of policy gradient methods even in an on-policy setting. The third assumption on lower bounded action probabilities holds by construction for the policy $\widetilde{\mu}$ due to the $\epsilon-$smoothing. The fourth assumption is about finite second momentum of distribution ratios. Prior works on off-policy evaluation in RL \citep{Jiang2016, wang2017optimal} both provide lower bounds for off-policy policy evaluation MSE in MDP and contextual bandit settings respectively that scale with the second moment of propensity scores, so this dependence is unavoidable for off-policy evaluation and learning.
Finally the regularity assumption on the smoothness of the reward function is again standard for policy gradient methods even in an on-policy setting.

Thus we find that under standard assumptions for policy gradient methods, along with some reasonable additional conditions, we expect \alg to have good convergence properties in theory.

\section{EXPERIMENTAL EVALUATION}
\label{sec:expts}
\begin{figure*}[!th]
\centering
    \begin{subfigure}[t]{0.46\textwidth}
    \includegraphics[height=2in]{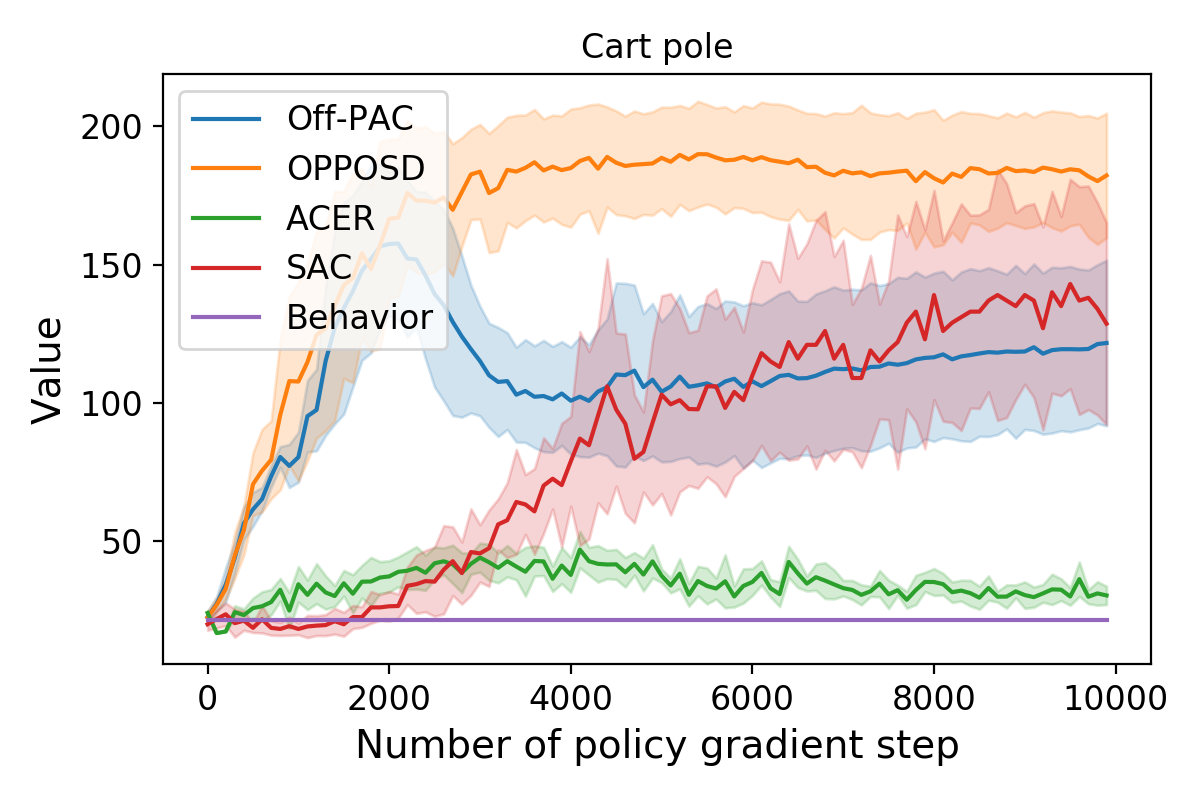} 
    \caption{CartPole-v0. Optimal policy gets a reward of 200, while the uniformly random data collection policy obtains 22.}
    \label{fig:cartpole}
    \end{subfigure}
    \qquad
    \begin{subfigure}[t]{0.46\textwidth}
    \includegraphics[height=2in]{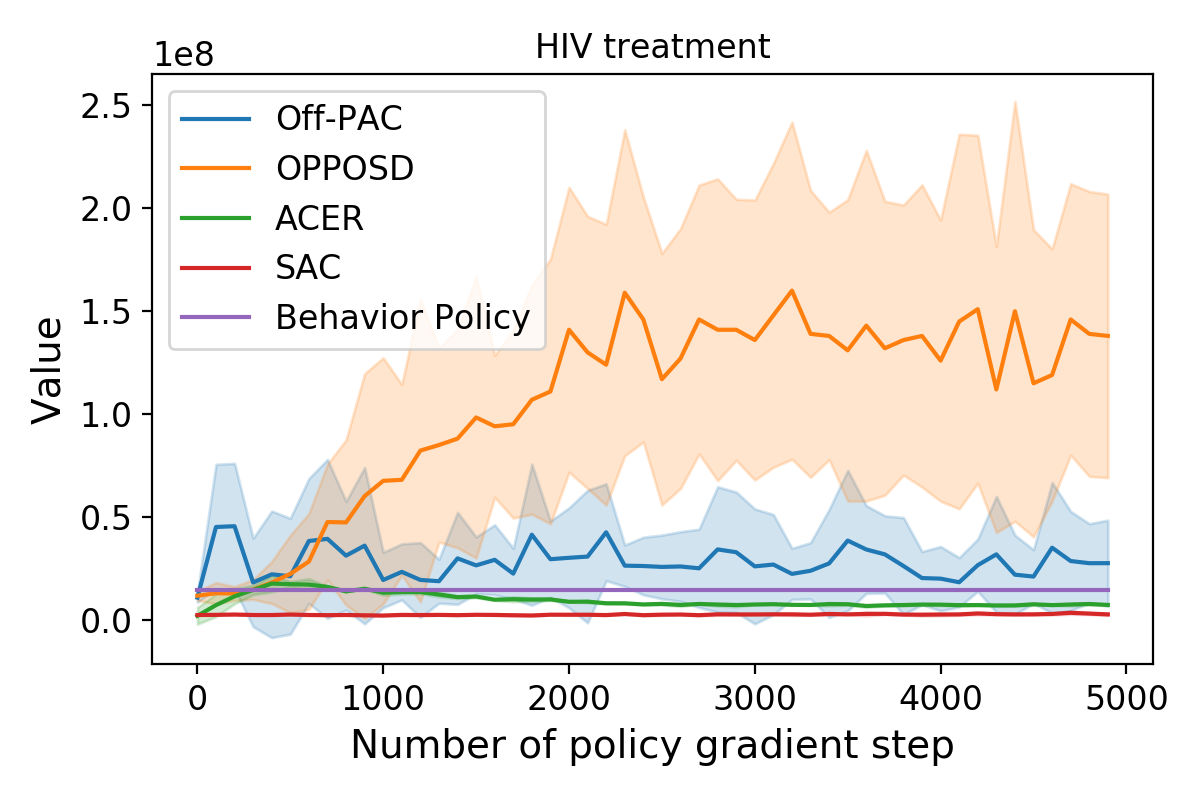} 
    \caption{HIV treatment simulator.  The data collection policy obtains a reward of 1.5E7.}
    \label{fig:hiv}
    \end{subfigure}
    \hfill
    \caption{Episodic scores over length 200 episodes in CartPole-v0~\citep{barto1983neuronlike,opanaigym} and HIV treatment simulator~\cite{ernst2006clinical}. Shaded region represents 1 standard deviation over 10 runs of each method.}
\end{figure*}
In this section we study the empirical properties of \alg, with an eye towards two questions:
\begin{enumerate}[nolistsep]
\item Does the state distribution correction help improve the performance of off-policy policy optimization?
\item Can we identify the best policy from the optimization path using off-policy policy evaluation?
\end{enumerate}

\paragraph{Baseline and implementation details} To answer the first question, we compare \alg with its closest prior work, but without the state distribution correction, that is the Off-PAC algorithm \citep{degris2012off}.

We implement both \alg and Off-PAC using feedforward neural networks for the actor and critic, with ReLU hidden layers. For state distribution ratio $w$, we also use a neural network with ReLU hidden layers, with the last activation function $\log ( 1 + \exp(x))$ to guarantee that $w(s) > 0$ for any input. To make a fair comparison, we keep the implementation of Off-PAC as close as possible to \alg other than the use of $w$. Concretely, we also equip Off-PAC with the enhancements that  improve empirical performance 
such as the warm start of the actor and critic, as well as several critic updates per actor update. We use the same off-policy critic learning algorithm for Off-PAC and \alg. To learn $w$, we use Algorithm 1 (average reward) in \citet{liu2018breaking} with RBF kernel for CartPole-v0 experiment, and Algorithm 2 (discounted reward) in \citet{liu2018breaking} with RBF kernel for HIV experiment. We normalize the input to the networks to have 0 mean and 1 standard deviation, and in each mini-batch we normalized kernel loss of fitting $w$ by the mean of the kernel matrix elements, to minimize the effect of kernel hyper-parameters on the learning rate. Full implementation details when omitted are provided in the Appendix \ref{sec:exp_detail}. We also compare \alg with two more recent off-policy policy gradient algorithms, namely soft actor-critic (SAC) and actor-critic with experience replay (ACER). 

\paragraph{Simulation domains}
We compare the algorithms in two simulation domains. The first domain is the \emph{cart pole} control problem, where an agent needs to balance a mass attached to a pole in an upright position, by applying one of two sideways movements to a cart on a frictionless track. The horizon is fixed to 200. If the trajectory ends in less than 200 steps, we pad the episode by continuing to sample actions and repeating the last state. We use a uniformly random policy to collect $n=500$ trajectories as off-policy data, which is a very challenging data set for off-policy policy optimization methods to learn from as this policy does not attain the desired upright configuration for any prolonged period of time. 

The second domain is an \emph{HIV treatment simulation} described in \cite{ernst2006clinical}. Here the states are six-dimensional real-valued vectors, which model the response of numbers of cells/virus to a treatment. Each action corresponds to whether or not to apply two types of drug, leading to a total of 4 actions. The horizon of this domain is 200 and discount factor is set by the simulator to $\gamma = 0.98$. A uniformly random policy does not visit any rewarding states often enough to collect useful data for off-policy learning. To simulate an imperfect but reasonable data collection policy, we first train an on-policy actor critic method to learn a reasonable (but still far from optimal) policy $\hat{\pi}$. We then use the data collection policy $\mu = 0.7*\hat{\pi} + 0.3*U$, where $U$ is the uniformly random policy, to collect $n=1000$ trajectories. 

Though in both domains our data collection policy is eventually able to cover the whole state-action space, the situation under finite amounts of data is different. In cart pole since an optimal policy can control the cart to stay in a small region, it is relatively easy for the uniform random policy to cover the states visited by the optimal policy. In the HIV treatment domain, it is unlikely that the logged data will cover the desirable state space. For the horizon of problem, theoretical derivation of our algorithm requires infinite horizon, but for computational feasibility, the horizon is truncated as 200 steps in the both domains.

\begin{figure*}[!t]
\centering
    \includegraphics[width=.45\textwidth]{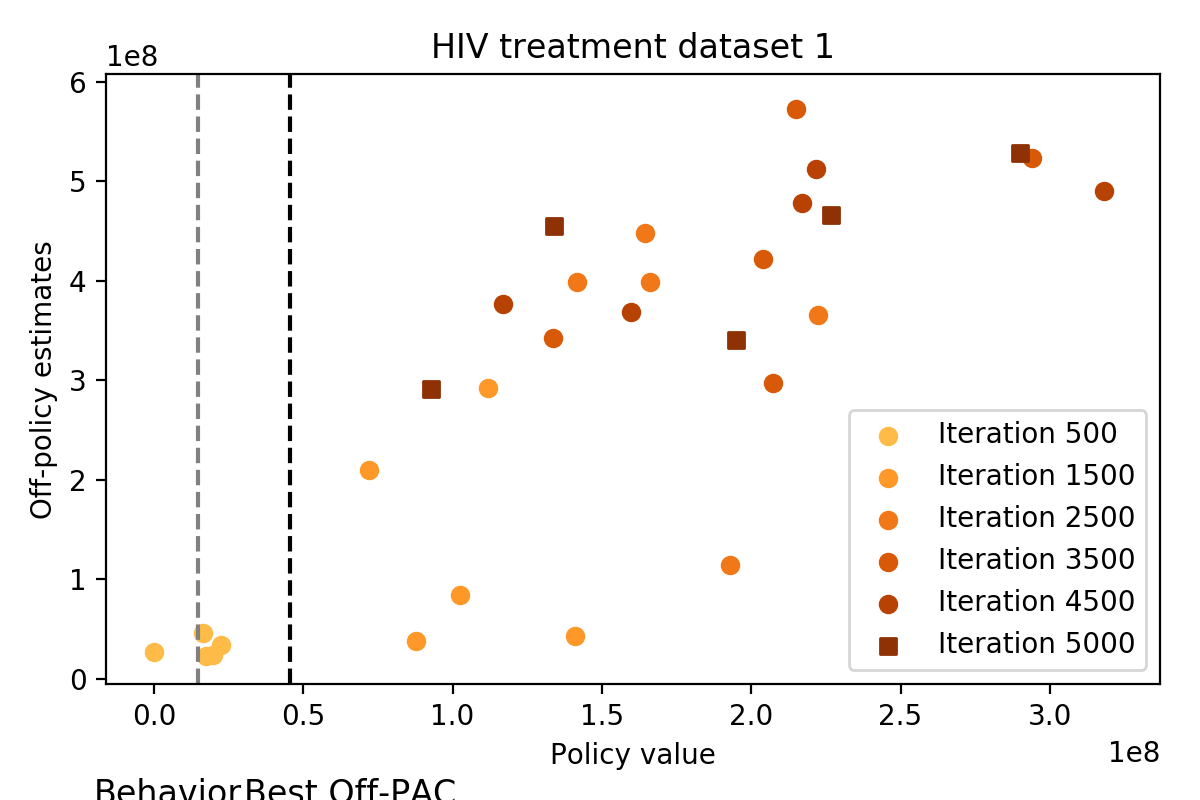} 
    \includegraphics[width=.45\textwidth]{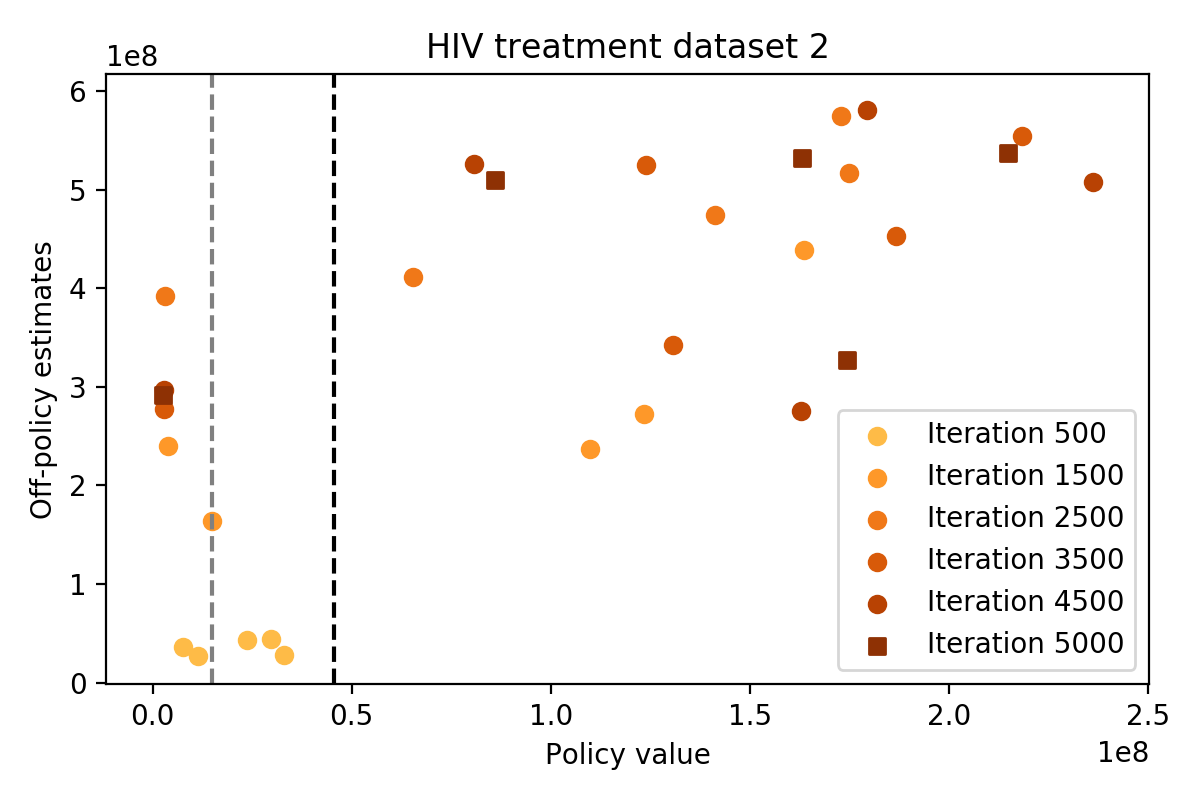}
    \caption{Off-policy policy evaluation results of saved policies from \alg. The estimated and true values exhibit a high correlation (coefficient = 0.80 and 0.71 in the left and right plots respectively) for most policies. Two panels correspond to repeating the whole procedure using two datasets from the same behavior policy.}
    \label{fig:hiv_oppe}
\end{figure*}

\paragraph{Impact of state reweighting on 
policy optimization}
In Figures \ref{fig:cartpole} and \ref{fig:hiv}, we plot the on-policy evaluation values of the policies produced by \alg and Off-PAC during the actor updates across 10 runs. Each run uses a different data set collected using the behavior policy as well as a different random seed for the policy optimization procedure. In each run we use the same dataset for each method to allow paired comparisons. We evaluate the policy after every 100 actor updates using on-policy Monte-Carlo evaluation over 20 trajectories. The results are averaged over 10 runs and error bars show the standard deviation. Along X-axis, the plot shows how the policy value changes as we take policy gradient steps.

At a high-level, we see that in both the domains our algorithm significantly improves upon the behavior policy, and eventually outperforms Off-PAC consistently. Zooming in a bit, we see that for the initial iterates on the left of the plots, the gap between \alg and Off-PAC is small as the state distribution between the learned policies is likely close enough to the behavior policy for the distribution mismatch to not matter significantly. This effect is particularly pronounced in Figure~\ref{fig:cartpole}. However, the gap quickly widens as we move to the right in both the figures. In particular, Off-PAC barely improves over behavior policy in Figure~\ref{fig:hiv}, while \alg finds a significantly better policy. Overall, we find that these results are an encouraging validation of our intuition about the importance of correcting the state distribution mismatch.

As for comparison with SAC and ACER, we would like to emphasize that we run SAC and ACER with a fixed batch of data which is not exactly the same setting as they were proposed for. In Appendix \ref{sec:appendix_related_work} we discuss more details on the difference of settings in this paper and SAC and ACER originally. As an illustration of the effect of this difference on algorithm performance, with the same function approximator class, both SAC and ACER perform worse than \alg in our experimental domains. Since SAC and ACER are not proposed for the batch off-policy settings and this simulation does not reflect these algorithms’ abilities in their proposed settings.

\paragraph{Identifying Best Policy by Off-Policy Evaluation}
While \alg consistently outperforms Off-PAC in average performance across 10 runs in both the domains, 
there is still significant variance in both the methods across runs. Given this variance, a natural question is whether 
we can identify 
the best performing policies, during and across multiple runs of \alg for a single dataset. To answer this question, 
we checkpoint all the policies produced by \alg after every 1000 actor updates, across 5 runs of our 
algorithm with the same input dataset generated in the HIV domain. 
We then evaluate these policies using the off-policy policy evaluation (OPPE) method in \citet{liu2018breaking}.  
The evaluation is performed with 
an additional dataset sampled from the behavior policy.
This corresponds to the typical practice of sample splitting between optimization and evaluation. 

We show the quality of the OPPE estimates against the true policy values for two different datasets for OPPE sampled from the behavior policy in the two panels of Figure \ref{fig:hiv_oppe}. In each plot, the X-axis shows the on-policy Monte-Carlo evaluation results and Y-axis shows the OPPE estimates. We find that the OPPE estimates are generally well correlated with the on-policy values, and picking the policy with the best OPPE estimate results in a true value substantially better than both the best Off-PAC result. A closer inspection also reveals the importance of this validation step. The red squares correspond to the final iterate of \alg in each of the 5 iterations, which has a very high value in some cases, but somewhat worse in other runs. Using OPPE to robustly select a good policy adds a layer of additional assurance to off-policy policy optimization procedure.

\section{DISCUSSION}
\label{sec:discussion}
We presented a new off-policy actor critic algorithm, \alg, based on a recently proposed stationary state distribution ratio estimator. There exist many interesting next steps, including different critic learning methods which may also leverage the state distribution ratio, and exploring alternative methods for policy evaluation or alternative stationary state distribution ratio estimators~\citep{hallak2017consistent,gelada2019off}. Another interesting direction is to improve the sample efficiency of online policy gradient algorithms by using our corrected gradient estimates. More discussion on concurrent potential future work is in Appendix \ref{sec:appendix_related_work}.

In parallel with our work, \citet{zhang2019generalized} have presented a different off-policy policy gradient approach. While similarly motivated by the bias in the Off-PAC gradient estimator, the two works have important differences. On the methodological side, \citet{zhang2019generalized} start from an off-policy objective function and derive a gradient for it. In contrast, we compute an off-policy estimator for the gradient of the on-policy objective function. The latter leads to a much simpler method, both conceptually and computationally, as we do not need to compute the gradients of the visitation distribution. On the other hand, \citet{zhang2019generalized} focus on incorporating more general interest functions in the off-policy objective, and use the emphatic weighting machinery. In terms of settings, our approach works in the offline setting (though easily extended to online), while they require an online setting in order to compute the gradients of the propensity score function. Finally, we present convergence results quantifying the error in our critic and propensity score computations while \citet{zhang2019generalized} assume a perfect oracle for both and rely on a truly unbiased gradient estimator for the convergence results. 

To conclude, our algorithm fixes the bias in off-policy policy gradient estimates introduced by the behavior policy's stationary state distribution. 
We prove under certain assumptions our algorithm is guaranteed to converge. We also show that ignoring the bias 
due to the mismatch in state distributions can make an off policy gradient algorithm fail even in a simple illustrative example, and that by 
accounting for this mismatch our approach yields significantly better performance in two simulation domains. 

\subsubsection*{Acknowledgements}
We acknowledge a NSF CAREER award, an ONR Young Investigator Award, and support from Siemens.

\bibliography{references}
\bibliographystyle{plainnat}

\appendix
\onecolumn

\section{Proof of Theorem \ref{thm:augmented_mdp}}
\label{sec:proof_thm1_appendix}
\begin{proof}
For any trajectory sampled from policy $\pi$, if every $s_k, a_k \in \mathcal{SA}_{\mu}$ then $\sum_{t=0}^{T} \gamma^t r(s_t,a_t) = \sum_{t=0}^{T} \gamma^t r_\mu(s_t,a_t)$. If not, let $s_{k+1},a_{k+1}$ be the first state-action pair that is not in $\mathcal{SA}_{\mu}$.
Then $\sum_{t=0}^{T} \gamma^t r(s_t,a_t) \ge \sum_{t=0}^{k} \gamma^t r(s_t,a_t) = \sum_{t=0}^{k} \gamma^t r_{\mu}(s_t,a_t) + \sum_{t=k+1}^{k} \gamma^t r_{\mu}(s_{abs},a_t)$.
Dividing the accumulated rewards by $\frac{1}{\sum_{t=0}^T \gamma^t}$, taking the limit of $T \to \infty$ and expectation over trajectories induced by $\pi$, we have that: $R^{\pi}_{M} \ge R^{\pi}_{M_{\mu}}$.
For $\pi^*$, since $\mathcal{SA}_{\mu}$ covers all state-action pairs reachable by $\pi^*$, so the expected return remains the same.
\end{proof}

\section{Proof of Theorem~\ref{thm:convergence}}
\label{sec:convergence}

We first state and prove an abstract result. Suppose we have a function $f~:~\R^d \to \R$ which is differentiable, $G$-Lipschitz and $L$-smooth, and $f$ attains a finite minimum value $f^* := \min_{x \in \R^d} f(x)$. Suppose we have access to a noisy gradient oracle which returns a vector $\grad(x) \in \R^d$ given a query point $x$. We say that the vector is $\sigma,B$-accurate for parameter $\sigma, B \geq 0$ if for all $x \in \R^d$, the quantity $\delta(x) := \grad(x) - \nabla f(x)$ satisfies

\begin{equation}
\|\E\left[ \delta(x) \mid x\right]\| \leq B \quad \mbox{and} \quad \E\left[\|\delta(x)\|^2 \mid x \right] \leq 2(\sigma^2 + B^2).
\label{eqn:grad_oracle}
\end{equation}

Notice that the expectations above are only with respect to any randomness in the oracle, while holding the query point fixed.
Suppose we run the stochastic gradient descent algorithm using the oracle responses, that is we update $x_{k+1} = x_k - \eta \grad(x_k)$. While several results for the convergence of stochastic gradient descent to a stationary point of a smooth, non-convex function are well-known, we could not find a result for the biased oracle assumed here and hence we provide a result from first principles. We have the following guarantee on the convergence of the sequence $x_k$ to an approximate stationary point of $f$.

\begin{theorem}
    Suppose $f$ is differentiable and $L$-smooth, and the approximate gradient oracle satisfies the conditions~\eqref{eqn:grad_oracle} with parameters $(\sigma_k, B_k)$ at iteration $k$. Then stochastic gradient descent with the oracle, with an initial solution $x_1$ and stepsize $\eta = 1/L$ satisfies after $K$ iterations:
    \[
    \frac{1}{K} \sum_{k=1}^K \E[\norm{\nabla f(x_k)}^2] \leq \frac{2}{K}(f(x_1) - f^*) + \frac{2}{LK}\sum_{k=1}^K(\sigma_k^2 + B_k^2).
    \]
\label{thm:apx_stationary}
\end{theorem}

\begin{proof}
    Since $f$ is $L$-smooth, we have
    \begin{align*}
        f(x_{k+1}) &\leq f(x_k) + \inner{\nabla f(x_k)}{x_{k+1} - x_k} + \frac{L}{2} \norm{x_{k+1} - x_k}^2\\
        &= f(x_k) - \eta\inner{\nabla f(x_k)}{\grad(x_k)} + \frac{L\eta^2}{2} \norm{\grad(x_k)}^2\\
        &= f(x_k) - \eta \inner{\nabla f(x_k)}{\delta(x_k) + \nabla f(x_k)} + \frac{L\eta^2}{2} \norm{\delta(x_k) + \nabla f(x_k)}\\
        &= f(x_k) + \norm{\nabla f(x_k)}^2 \left(\frac{L\eta^2}{2} - \eta\right) - (\eta - L\eta^2) \inner{\nabla f(x_k)}{\delta(x_k)} + \frac{L\eta^2}{2} \norm{\delta(x_k)}^2.
    \end{align*}
    Here the first equality follows from our update rule while the remaining simply use the definition of $\delta$ along with algebraic manipulations. Now taking expectations of both sides, we obtain
    \begin{align*}
        \E[f(x_{k+1})] &\leq \E[f(x_k)] + \E[\norm{\nabla f(x_k)}^2] \left(\frac{L\eta^2}{2} - \eta\right) + (\eta - L\eta^2)GB_k + L\eta^2(\sigma_k^2 + B_k^2),
    \end{align*}
    where we have invoked the properties of the oracle to bound the last two terms. Summing over iterations $k = 1,2,\ldots, K$, we obtain

    \begin{align*}
         \E[f(x_{k+1})] \leq f(x_1) + \left(\frac{L\eta^2}{2} - \eta\right)\sum_{k=1}^K \E[\norm{\nabla f(x_k)}^2] + \eta \sum_{k=1}^K(GB_k(1 - L\eta) + L\eta(\sigma_k^2 + B_k^2)).
    \end{align*}
    Rearranging terms, and using that $f(x_{K+1}) \geq f^*$, we obtain

    \begin{align*}
        \frac{1}{K} \sum_{k=1}^K \E[\norm{\nabla f(x_k)}^2] \leq \frac{f(x_1) - f(x^*)}{K(\eta - L\eta^2/2)} + \frac{\eta \sum_{k=1}^K(GB_k(1 - L\eta) + L\eta(\sigma_k^2 + B_k^2))}{K(\eta - L\eta^2/2)}.
    \end{align*}
    Now using the choice $\eta = 1/L$ and simplifying, we obtain the statement of the theorem.
\end{proof}

The theorem tells us that if we pick an iterate uniformly at random from $x_1,\ldots, x_K$, then it is an approximate stationary point in expectation, up to an accuracy which is determined by the bias and variance of the stochastic gradient oracle.

Given this abstract result, we can now prove Theorem~\ref{thm:convergence} by instantiating the errors in the gradient oracle as a function of our assumptions. The proof is similar with convergence proof of stochastic gradient descent from \cite{ghadimi2013stochastic}, while they do not actually cover the case of biased errors in gradients, which is the reason a first-principles proof was included here.

\paragraph{Proof of Theorem~\ref{thm:convergence}}
\label{sec:proof_thm3_appendix}
We now instantiate the result and assumptions for the case of the off-policy policy gradient method. First, note that the algorithm is stochastic gradient ascent for maximizing the expected return $J(\theta) := R^{\pi_\theta}$. Thus we can apply Theorem~\ref{thm:apx_stationary} with $f = -J$, so that $f(x_1) - f^* \leq V_{\max}$ where $V_{\max}$ is an upper bound on the value of any policy in the MDP. $f$ attains a finite minimum value since the expected return has a finite maximum value. We focus on quantifying the bias $B$ in terms of errors in the critic and propensity score computations first. We first introduce some additional notation. Suppose $w_\theta(s)$ and $Q^\theta(s,a)$ are the true propensity (in terms of state distributions, relative to $\mu$) and $Q$-value functions for a policy $\pi_{\theta}$. Let $g_\theta(s,a) = \frac{\partial \log \pi_{\theta}(a | s)}{\partial \theta}$. Suppose we are given estimators $\wh$ and $\Qh$ for $w_\theta$ and $Q^\theta$ respectively. Then our estimated and true off-policy policy gradients can be written as:
\[
    \nabla_\theta J(\theta) = \E_{\nu} [w\rho g_\theta Q^\pi] \quad \mbox{and} \quad \grad(\theta) = \wh\rho g_\theta \Qh.
\]

Now the bias can be bounded as
\begin{align*}
    \norm{\E[\grad(\theta) - \nabla J(\theta) | \theta]} &= \norm{\E_\nu[ w\rho g_\theta Q^\theta - \wh\rho g_\theta \Qh]}\\
    &\leq \norm{\E_\nu[(w - \wh)\rho g_\theta Q^\theta]} + \norm{\E_\nu[\wh \rho g_\theta (Q^\theta - \Qh)]} \\
    &\leq \E_\nu[\norm{(w - \wh)\rho g_\theta Q^\theta}] + \E_\nu[\norm{\wh \rho g_\theta (Q^\theta - \Qh)}] \\
    &\leq \E_\nu[ |w - \wh| \norm{\rho g_\theta} |Q^\theta|] + \E_\nu[|\wh|  \norm{\rho g_\theta} |Q^\theta - \Qh|]
\end{align*}

By Assumption \ref{assum:convergence} we have that
\begin{align*}
    \norm{\rho(s,a) g_\theta(s,a)} = \norm{\frac{1}{\nu(a|s)}\frac{\partial \pi_{\theta}(a|s)}{\partial \theta}} \le \frac{\gmax}{\numin},
\end{align*}
Then the bound on the bias simplifies to
\begin{align*}
\norm{\E[\grad(\theta) - \nabla J(\theta) | \theta]} &\leq \frac{\gmax}{\numin} \E_\nu[ |w - \wh| |Q^\theta|] + \frac{\gmax}{\numin} \E_\nu[|\wh| |Q^\theta - \Qh|]
\end{align*}

How we simplify further depends on the assumptions we make on the errors in $\wh$ and $\Qh$. As a natural assumption, suppose that the relative errors are bounded in MSE, that is $\E_{\nu}(w(s) - \wh(s))^2 \leq \eps^2_w$ and $\E_{\nu}\left(Q^\theta(s,a) - \Qh(s,a)\right)^2 \leq \eps^2_Q$. Then by Cauchy-Shwartz inequality, we can simplify the above bias term as
\begin{align*}
\norm{\E[\grad(\theta) - \nabla J(\theta) | \theta]} &\leq \frac{\gmax}{\numin} \eps_w \sqrt{\E_\nu [(Q^\theta)^2]}
+ \frac{\gmax}{\numin} \eps_Q \sqrt{\E_\nu [\wh^2]}
\end{align*}

By Assumption \ref{assum:convergence} we have  $Q^{\theta} \leq \vmax$ for all $s,a$, and
\begin{align*}
    \E_\nu [\wh^2] &\le \E_\nu [ w(s)^2 + (w(s) - \wh(s))^2 ] \le \E_\nu [ w(s)^2] + \E_\nu [ (w(s) - \wh(s))^2 ] \le \wmax^2 + \eps_w^2
\end{align*}

Then the bound on the bias further simplifies to
\begin{align*}
\norm{\E[\grad(\theta) - \nabla J(\theta) | \theta]} &\leq \eps_w \gmax \vmax /\numin + \eps_Q \gmax \sqrt{\wmax^2 + \eps_w^2} /\numin
\end{align*}

Similarly, for the variance we have
\begin{align*}
\E[\norm{\grad(\theta) - \nabla J(\theta)}^2 | \theta] &\leq 2 \E_\nu\left[\norm{(\wh - w) \rho g_\theta Q^\theta}^2\right] + 2\E_{\nu}\left[\norm{\wh \rho g_\theta (\Qh - Q^\theta)}^2\right]\\
&\leq 2\eps_w^2 \gmax^2\vmax^2/\numin^2 + 2\eps_Q^2 (\wmax^2 + \eps_w^2)\gmax^2/\numin^2.
\end{align*}

Hence, the RHS of Theorem~\ref{thm:apx_stationary} simplifies to
\[
\frac{2\vmax}{K} + O\left(\frac{\sum_{k=1}^T \eps_{w,k}^2 \gmax^2\vmax^2/\numin^2 + \eps_{Q,k}^2 (\wmax^2 + \eps_{w,k}^2)\gmax^2/\numin^2}{K}\right),
\]
where $\eps_{w,k}$ and $\eps_{Q,k}$ are the error parameters in the propensity scores and critic at the $k_{th}$ iteration of our algorithm. Since we update these quantities online along with the policy parameters, we expect $\eps_{w,k}$ and $\eps_{Q,k}$ to decrease as $k$ increases.
That is, assuming that $\nu$ satisfies the coverage assumptions with finite upper bounds on the propensities and the policy class is Lipschitz continuous in its parameters, the scheme converges to an approximate stationary point given estimators $\wh$ and $\Qh$ with a small average MSE across the iterations under $\nu$.

\section{Details for Hard Example of \offpac}
\label{sec:hard_example}
Consider the problem instance shown in Figure \ref{fig:hard_example}, and the policy class describe in the paper. The true state value function of $\pi_\alpha$, $V_{\pi_\alpha}$ satisfies that:
$
    V^{\pi_\alpha}(s_0) = V^{\pi_\alpha}(s_1) = \tfrac{1+\alpha}{2}, ~~V^{\pi_\alpha}(s_2) = \tfrac{1-\alpha}{2}, ~~ V^{\pi_\alpha}(s_3) = 1,~~V^{\pi_\alpha}(s_4) = 0.
$
We now study the Off-PAC gradient estimator $\goffpac(\alpha)$ in an idealized setting where the critic $Q^{\pi_\alpha}$ is perfectly known. As per Equation 5 of~\citet{degris2012off}, we have
\begin{align*}
    \goffpac(\alpha) &= \sum_s d^{\pi_b}(s) \sum_a \frac{\partial \pi_\alpha(a|s)}{\partial \alpha} Q^{\pi_\alpha}(s,a)\\
    &= d^{\pi_b}(s_1) \bigg(Q^{\pi_\alpha}(s_1,\ell) - Q^{\pi_\alpha}(s_1,r)\bigg) \\
    & + d^{\pi_b}(s_2) \bigg(Q^{\pi_\alpha}(s_2,\ell) - Q^{\pi_\alpha}(s_2,r)\bigg) \\&= \tfrac{1}{2}(1-1/2) + \tfrac{1}{2}(0-1/2) = 0.
\end{align*}
That is, the gradient vanishes for any policy $\pi_\alpha$, meaning that the algorithm can be arbitrarily sub-optimal at any point during policy optimization. We note that this does not contradict the previous Off-PAC theorems as the policy class is not fully expressive in our example, a requirement for their convergence results. 

\begin{figure}[ht]
\centering
  \includegraphics[height=1.5in]{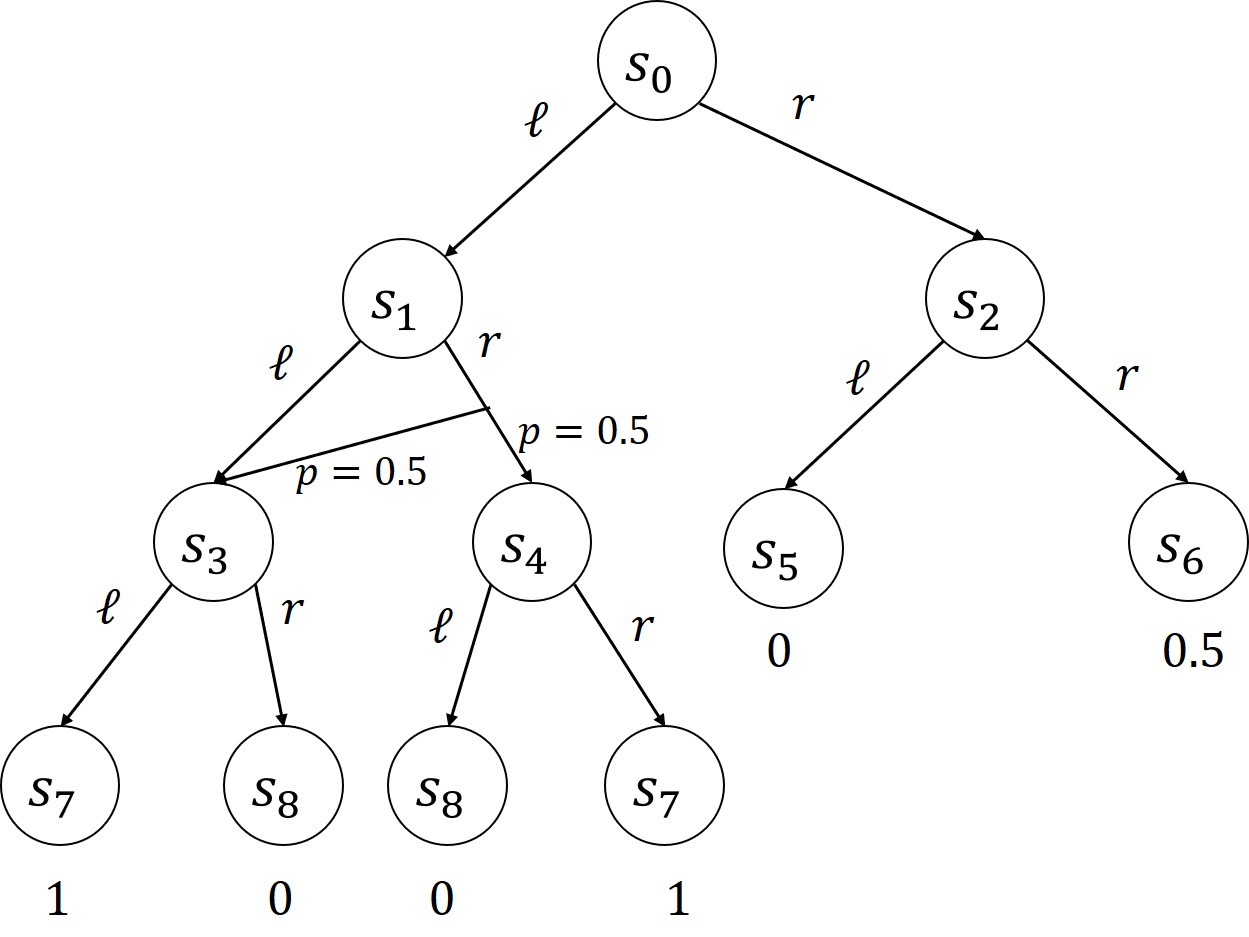}
\end{figure}

\section{Details for Experiments}
\label{sec:exp_detail}
In this section we will show some important details and hyper-parameter settings of our algorithm in experiments. 

\subsection{Simulation Domains} For the first domain \emph{cart pole} control problem, the state space is continuous and describes the position and velocity of cart and pole. The action space consists of applying a unit force to two directions. The horizon is fixed to 200. If the trajectory ends in less than 200 steps, we pad the episode by continuing to sample actions and repeating the last state. We use a uniformly random policy to collect $n=500$ trajectories as off-policy data. We use neural networks with a 32-unit hidden layer to fit the stationary distribution ratio, actor and critic.

The second domain is an \emph{HIV treatment simulation} described in \cite{ernst2006clinical}. 
The transition dynamics are modeled by an ODE system in \citet{ernst2006clinical}. The reward consists of a small negative reward for deploying each type of drug, and a positive reward based on the HIV-specific cytotoxic T-cells which will increase with a proper treatment schedule. To maximize the total reward in this simulator, algorithms need to do structured treatment interruption (STI), which aim to achieve a balance between treatment and the adverse effect of abusing drugs. 
Each trajectory simulates a treatment period for one patient in 1000 days and each step corresponds to a 5-day interval in the ODE system. We represent the state by taking logarithm of state features and divide the reward by $10^8$ to ensure they are in a reasonable range to fit the neural network models. 
We use neural networks with three 16-unit hidden layers to fit the actor and state distribution ratio, and a neural network with four 32-unit hidden layers for the critic.

\subsection{Hyper-parameters} We use three separate neural networks, one for each of actor, critic, and the state distribution ratio model $w$. We use the Adam optimizer for all of them. We also use a entropy regularization for the actor. We warm start the actor by maximizing the log-likelihood of actor on the collected dataset. For critic, we use the same critic algorithm as we used in Algorithm \ref{alg:opposd} except that there is no importance sampling ratio (as it is on-policy for the warm start). For the warm start of w, we just fit the w for several iterations using the warm start policy found for the actor. Warm start uses the same learning rates as normal training. For critic and $w$, we also keep the state of optimizer to be the same when we start normal training.

\begin{table}[ht!]
\begin{center}
\begin{tabular}{c|cc}
 Hyper-parameters  & cart pole  & HIV  \\
 \hline
 $\gamma$ & 1.0 &0.98 \\
 $\lambda$ & 0. & 0. \\
 entropy coefficient & 0.01 & 0.03 \\
 learning rate (actor)  & 1e-3  & 5e-6   \\
 learning rate (critic)  & 1e-3  & 1e-3   \\
 learning rate (w) & 1e-3 & 3e-4 \\
 batch size (actor)  & 5000 & 5000  \\
 batch size (critic)  & 5000 &  5000 \\
 batch size (w)  & 200 & 200  \\
 number of iterations (critic)  & 10 & 10  \\
 number of iterations (w)  & 50 & 50 \\
 weight decay (w) & 1e-5 & 1e-5 \\
 behavior cloning number of iterations & 2000 & 2000 \\
 warm start number of iterations (crtic)  & 500 & 2500 \\
 warm start number of iterations (w)  & 500 & 2500 \\
\end{tabular}
\caption{Hyper-parameters in experiments}
\label{tab:hyperparameters}
\end{center}
\end{table}

In the Table \ref{tab:hyperparameters} we show some hyper-parameters setting we used in both domain. We also follow the details in Algorithm 1 and Algorithm 2 of \citet{liu2018breaking} to learn $w$. We scale the inputs to $w$ so that the whole off-policy dataset has zero mean and standard deviation of 1 along each dimension in state space. We use the RBF kernel to compute the loss function for $w$. For the cart pole simulator, the kernel bandwidth is set to be the median of state distance. If computing this median state distance over the whole off-policy dataset is computationally too expensive, it can be approximated using a mini-batch. In the HIV domain the bandwidth is set to be 1. When we compute the loss of $w$, we need to sample two mini-batch independently to get an unbiased estimates of the quadratic loss. The loss in each pair of mini-batch is normalized by the sum of kernel matrix elements computed from them.

\subsection{Choice of Algorithm with Discounted Reward}
In discounted reward settings, the state distribution is also defined with respect to the discount factor $\gamma$, and \citet{liu2018breaking} introduce an algorithm to learn state distribution ratio in this setting. However, we notice that in on-policy policy learning cases, though the policy gradient theorem \citep{sutton2000policy} requires samples from the stationary state distribution defined using the discount factor, it is common to directly use the collected samples to compute policy gradient without re-sampling/re-weighting $(s,a,r,s')$'s according to the discounted stationary distribution. This might be driven by sample efficiency concerns, as  samples at later time-steps in the discounted stationary distribution will receive exponentially small probability, meaning they are not leveraged as effectively by the algorithm. Given this, we compare three different variants of our algorithm in HIV experiment with discounted reward. The first (OPPOSD average) variant uses the algorithm for the average reward setting, but evaluates its discounted reward. The second learns the state distribution ratio $w(s)$ in the discounted case (Algorithm 2 in \citep{liu2018breaking}), but still samples from the undiscounted distribution to compute the gradient (OPPOSD disc $w$). The third learns the state distribution ratio $w(s)$ in the discounted case and also re-samples the samples according to $
    d^{\pi}(s) = \displaystyle\lim_{T\to\infty}\textstyle\frac{1}{\sum_{t=0}^T \gamma^t} \sum_{t=0}^{T} \gamma^t d^{\pi}_{t}(s)
$ (OPPOSD). In the main body of paper, we select the third one as it is the most natural way from the definition of problem and policy gradient theorem. Results of these three methods are demonstrated in Figure \ref{fig:hiv_all} and they do not have significant differences in this experiment.
\begin{figure}[ht]
    \centering
    \includegraphics[height=2in]{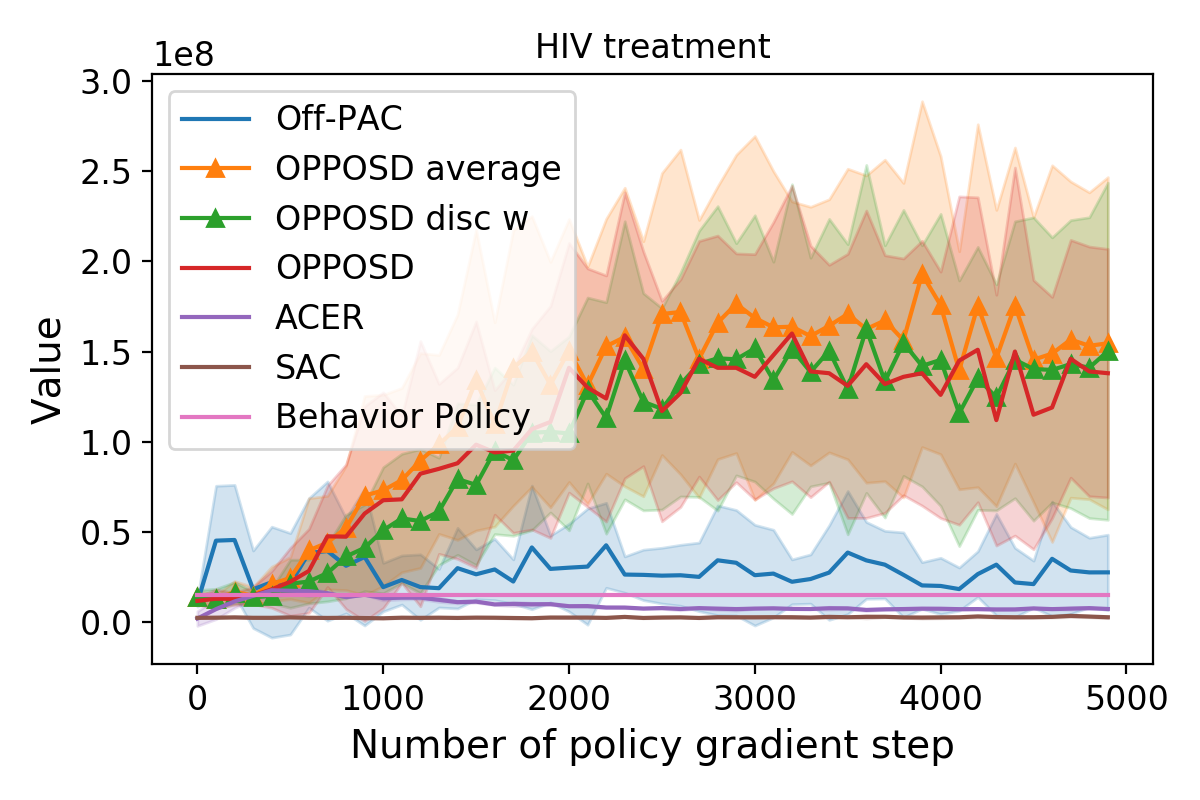} 
    \caption{Episodic scores over length 200 episodes in HIV treatment simulator.}
    \label{fig:hiv_all}
\end{figure}

\section{Discussion on Related Off-Policy Learning Algorithms}
\label{sec:appendix_related_work}

There are also many different algorithms which have been built using \offpac \citep{degris2012off} and improve \offpac in different directions, such as DDPG \citep{lillicrap2015continuous}, ACER \citep{wang2016sample}, etc. They are orthogonal to our work and our state distribution correction techniques are composable with these further improvements in the \offpac framework. For understanding the impact of correcting the stationary distribution, in the experiment section of this work we therefore focus on ablation comparison with \offpac. It would be interesting to combine our work with the additional contributions of DDPG, ACER etc. to derive improved variants of each of those algorithms. 

We also wish to clarify that some previous off-policy policy gradient algorithms such as DDPG, SAC \citep{Haarnoja2018Soft} and ACER focus on a different setting with this paper -- they consider online off policy where data is collected every iteration using the current policy with potentially noise, and the off policy nature comes from when updating to a new policy, the algorithms use all the data collected across previous iterations. In contrast our focus and experiments are on batch off policy learning, where data is collected in advance from a behavior policy and a new policy is computed using only that batch dataset.  The difference between the online off-policy and batch off-policy settings is important since algorithms that receive periodic access to new samples gathered using the current policy may benefit significantly. To illustrate the difference in the two settings, we ran SAC and related ACER in our batch off-policy setting in the experimental section. SAC is proposed for continuous actions setting in their paper and available code, and we re-derive the policy gradient updates for discrete actions. 

We separate out discussion of SBEED \citep{Dai2018SBEED} since its theoretical results are derived for a similar batch off policy RL setting as our approach. However, it has not been experimentally evaluated in a batch off-policy setting (their empirical results were for off policy RL as described above where more data is collected). SBEED advances over a number of prior theoretical results on sample complexity for value function approaches for batch off policy learning \citep{Antos2008Learning}. In contrast, in policy gradient methods there has not even existed a statistically consistent procedure for batch off-policy learning, and this is the fundamental contribution of our work. In many domains policy optimization techniques are the methods of choice, and policy improvement from a reasonable policy is often more natural in safety critical applications. Since these are the types of applications where off-policy carries the most appeal, we focus on the class of policy optimization algorithms that can work in a batch off-policy setting. We can possibly leverage similar techniques in value function learning too, and we view that development along with a detailed evaluation against other value function learning methods as future work.

\end{document}